\documentclass[11pt]{article}
\usepackage[T1]{fontenc}
\usepackage{lmodern}
\usepackage[margin=1in]{geometry}
\usepackage{amsmath,amssymb,amsthm,mathtools}
\usepackage{enumitem}
\usepackage{array,booktabs,tabularx}
\usepackage{microtype}
\usepackage{graphicx,xcolor}
\usepackage[section]{placeins}
\usepackage[authoryear,round]{natbib}
\usepackage{tikz}
\usetikzlibrary{arrows.meta,positioning,calc,fit,backgrounds}
\usepackage[hidelinks]{hyperref}
\hypersetup{pdftitle={Expansion Counts under Standard A* Tie-Breaking Strategies on the Final Plateau},
  pdfauthor={Alex Fukunaga}}  
\tikzset{
  state/.style={circle,draw,thick,minimum size=7.2mm,inner sep=0pt,font=\small},
  start/.style={state,fill=gray!8},
  goal/.style={state,double,double distance=1.2pt},
  dead/.style={state,fill=gray!12},
  edge/.style={-{Latex[length=2mm]},semithick},
  edgelabel/.style={font=\scriptsize,fill=white,inner sep=1.2pt},
  plateau/.style={draw=gray!60,dashed,rounded corners,inner sep=4pt},
  ell/.style={font=\large},
  note/.style={font=\scriptsize,align=center,inner sep=1pt},
  gridfree/.style={draw=gray!55,fill=white,line width=0.35pt},
  gridbad/.style={draw=gray!55,fill=gray!18,line width=0.35pt},
  gridblock/.style={draw=gray!55,fill=gray!65,line width=0.35pt},
  gridcallout/.style={font=\scriptsize,align=center,fill=white,inner sep=1.5pt}
}
\title{Expansion Counts under Standard A* Tie-Breaking Strategies on the Final Plateau}
\author{Alex Fukunaga \\ The University of Tokyo}

\date{}
\theoremstyle{plain}
\newtheorem{theorem}{Theorem}
\newtheorem{proposition}{Proposition}
\newtheorem{lemma}{Lemma}
\newtheorem{corollary}{Corollary}
\theoremstyle{definition}

\newcommand{\astar}{A$^*$}
\newcommand{\Open}{\ensuremath{\mathrm{OPEN}}}
\newcommand{\Closed}{\ensuremath{\mathrm{CLOSED}}}
\newcommand{\OpenMin}{\ensuremath{\mathrm{OPEN}_{\min}}}
\newcommand{\gstar}{g^*}
\newcommand{\hstar}{h^*}
\newcommand{\Cstar}{C^*}
\newcommand{\FIFO}{\ensuremath{\mathrm{FIFO}}}
\newcommand{\LIFO}{\ensuremath{\mathrm{LIFO}}}
\newcommand{\RAND}{\ensuremath{\mathrm{RANDOM}}}
\newcommand{\LowH}[1]{\ensuremath{[h,#1]}}
\newcommand{\HighH}[1]{\ensuremath{[-h,#1]}}
\newcolumntype{Y}{>{\raggedright\arraybackslash}X}

\usepackage{comment}

\begin{document}
\maketitle

\begin{abstract}
In the A* search algorithm, the tie-breaking strategies for  nodes with the same $f$-value determines which states A* expands on the final $f$-layer.
For nine standard tie-breaking strategies, we show that under a consistent heuristic, every pair has positive-cost instances
favoring each strategy over the other by an arbitrarily large additive expansion gap.
A parameterized unit-cost grid example also gives unbounded expansion-count ratios between low-$h$ with FIFO and LIFO.
In unit-cost search with $h > 0$ at non-goals, exact heuristic values near the goal lead to complementary extremal results: low-$h$ minimizes the number of remaining expansions from a common configuration within the perfect region, while high-$h$ maximizes the total number of expansions when every final-plateau state with $h=1$ is a goal predecessor.
Finally, with the evaluation function $f_{\alpha} = g + \alpha h$, when $h>0$ at non-goals,  every heuristic weight $0 \leq \alpha<1$ eliminates tie-breaking sensitivity, and all tie-breaking strategies expand the same set of states.

\end{abstract}

\section{Introduction}\label{sec:introduction}

The \astar{} search algorithm for finding an optimal-cost path from a start state to a goal state in a graph orders states according to the evaluation function  $f=g+h$, where $g$ is the cost of the current path
from the initial state and $h$ estimates the remaining cost to a goal
\citep{hart1968}.
Since multiple states in \Open{} may have the same minimum $f$-value,
a complete algorithm specification requires a sequence of \emph{tie-breaking} criteria.
A common choice prefers smaller $h$ and resolves any remaining
ties by insertion order (FIFO, LIFO) or by uniform random selection.  On an equal-$f$ tie, preferring smaller $h$ is
equivalent to preferring larger $g$.

For optimal search, previous work has investigated the performance of tie-breaking strategies for \astar. While \astar\ must expand all nodes with $f < C^*$, where $C^*$ is the cost of an optimal solution, the number of nodes expanded in the final plateau (nodes with $f = C^*$) is affected by the tie-breaking strategy. 
\citet{asai2016aaai,asai2017jair} experimentally evaluated the performance of 
FIFO, LIFO, random, high-$g$, and low-$g$ tie-breaking policies and proposed plateau-depth-based rules . \citet{correa2018optimal}  proposed tie-breaking based on cost adaptation.

Despite the long history of \astar{} and the wide acknowledgment that tie-breaking affects search efficiency, to our knowledge, formal analysis of tie-breaking has received relatively little attention.
\citet{correa2018optimal} showed that optimal tie-breaking strategies can expand fewer nodes than some common tie-breaking strategies. However, pairwise relationships among the most commonly used tie-breaking strategies which are based on $g$-values and order of entry into \Open{} have not been systematically analyzed.

We analyze expansion counts under standard tie-breaking strategies based on $h$-values and insertion order into \Open{}. We make three contributions.

The first contribution is a %
pairwise incomparability result.
For nine standard strategies formed from FIFO, LIFO, uniform random
selection, and low-$h$ or high-$h$ criteria, a single parameterized
construction with a consistent heuristic and positive edge costs yields arbitrarily large expansion gaps in both directions.
The deterministic comparisons hold for expanded sets as well as counts;
comparisons involving randomization use exact expectations.
In other words, none of these
nine standard tie-breaking rules dominate the others in general.
Unbounded gaps among standard tie-breaking rules can be found even in familiar domains, as a parameterized example on four-neighbor unit-cost grids with
Manhattan distance and a fixed successor order gives unbounded
additive gaps and expansion-count ratios between \LowH{\FIFO}
and \LowH{\LIFO}. Transposing the grid reverses which strategy
expands fewer states.

The second contribution identifies extremal strategies under a local
accuracy condition on the heuristic.  In unit-cost search with $h > 0$ at non-goals, suppose that $h$ is
perfect on final-plateau states with $h\leq k$.  From a common search
configuration whose minimum eligible $h$-value is $m\leq k$, every low-$h$
strategy performs exactly $m$ further expansions, and is minimal, i.e., no continuation
performs fewer expansions.
The guarantee concerns the \emph{remaining} expansions;
when the condition holds at entry to the final plateau, the continuation guarantee yields global expansion optimality.
Furthermore, under the weaker condition that every
final-plateau state with $h=1$ is a goal predecessor, every high-$h$
strategy attains the maximum \emph{total} expansion count, independently of
its final tie-breaker.  %

Finally, we elucidate how the tie-breaking-dependent search on the final plateau $f=C^*$ arises, by considering search using the weighted evaluation function $f_{\alpha} = g + \alpha h$.  When $h>0$ at non-goals,  every underweighted heuristic search with  weight $0 \leq \alpha<1$ expands the same set of states, regardless of tie-breaking,
The common expanded set for $0<\alpha<1$ contains the states expanded by ordinary \astar{} and is contained in the uniform-cost expanded set, so the disappearance of tie-breaking sensitivity comes from making formerly optional final-layer states compulsory.
Thus, there is a sharp boundary at the standard \astar{} weight $\alpha=1$, 
which  restores a final layer on which tie-breaking can affect the expanded set; indeed, as shown by the first contribution, the resulting expansion gap can be arbitrarily large.

\section{Preliminaries and Previous Expansion Bounds}\label{sec:model}

We begin by describing the graph-search model, tie-breaking, and comparison criteria, then recalling the
classical expansion bounds used in the subsequent analysis.

\subsection{Search model}

A search instance is $I=(V,A,c,s,\mathcal G,h)$, where $(V,A)$ is a finite
directed graph, $c:A\to\mathbb R_{\geq0}$ assigns transition costs,
$s\notin\mathcal G$ is the initial state, and $\mathcal G$ is the set of
goals. The heuristic $h:V\to\mathbb R$ is state-based. Let $\gstar(v)$
be the least cost of a path from $s$ to $v$, and let $\hstar(v)$ be the
least cost from $v$ to a goal; either is $+\infty$ if no such path exists.
We assume that the instance is solvable and write
$\Cstar=\hstar(s)<+\infty$.

A heuristic $h$ is \emph{perfect} if $h(v)=\hstar(v)$ for every
state $v$. More generally, $h$ is \emph{perfect on a set
$W\subseteq V$} if $h(v)=\hstar(v)$ for every $v\in W$.

Unless stated otherwise, $h$ is consistent,
\[
  h(u)\leq c(u,v)+h(v)\qquad ((u,v)\in A),
\]
and $h(G)=0$ for every goal $G$. These conditions imply admissibility,
$h(v)\leq\hstar(v)$. Positivity at non-goals is assumed only where stated.
Successors are generated in a fixed order, which is part of the
instance and is identical in every comparison on that instance.

\astar{} maintains \Open{} and \Closed{}, with at most one current
\Open{} entry per state. An entry stores its best known $g$-value,
parent, first-insertion timestamp, and goal status. Expanding a non-goal
state removes it from \Open{}, adds it to \Closed{}, and generates all
its successors. A newly encountered state is goal-tested and inserted.
A strictly smaller tentative $g$ updates an existing entry and its
parent, but \emph{preserves its first-insertion timestamp}; equal- and
higher-cost duplicates are discarded. Consistency ensures that an
expanded state has $g=\gstar$ and is never reopened. Stale priority-queue
records, if retained by an implementation, are ignored and do not count
as expansions.

The goal test is evaluated at generation, including for the initial
state. Search terminates only when a goal is selected from \Open{}.
At a selection point, define
\[
  \OpenMin=\operatorname*{arg\,min}_{v\in\Open}\{g(v)+h(v)\}.
\]
If \OpenMin{} contains a goal, a goal is selected immediately; otherwise
a tie-breaker selects a non-goal in \OpenMin{} for expansion. We call the states of \OpenMin{} \emph{eligible} at that selection point. Selecting 
a goal is not counted as an expansion. A generated goal with larger $f$
remains in \Open{} and does not cause termination.

This goal-preference convention is standard in analyses of \astar{}
tie-breaking \citep{dechter1985optimality,correa2018optimal}.
Adding it to a minimum-$f$ search cannot increase its expansion count:
under otherwise identical choices, the two runs share a prefix, after
which the goal-preferring run may terminate while the other continues.

Our comparisons hold the graph, heuristic, successor order, duplicate
handling, and goal-preference rule fixed. The importance of specifying
all levels of tie-breaking is also emphasized by
\citet{barley2025methodology}, who show that incomplete descriptions can
lead to irreproducible and misleading comparisons in optimal
bidirectional heuristic search. Here, these conventions isolate the
expansion effects of the tie-breaking rules being compared.

\subsection{Tie-breaking and comparisons}

Tie-breaking criteria are applied lexicographically after $f$ and goal
preference. We omit these two common criteria when naming a strategy.
Thus \LowH{\FIFO} denotes low-$h$ followed by FIFO, and
\HighH{\LIFO} denotes high-$h$ followed by LIFO. The full ordering for
\LowH{\FIFO} is $[f,\text{goal preference},h,\FIFO]$.
FIFO and LIFO select the earliest and latest first-insertion timestamps,
respectively. RANDOM samples uniformly from the states left tied by all
preceding criteria, independently at each selection. The terms
\emph{low-$h$} and \emph{high-$h$} name criteria, whereas
\LowH{\tau} and \HighH{\tau} name strategies with a specified final
tie-breaker $\tau$.

At equal $f$, minimizing $h$ is equivalent to maximizing $g$, and
maximizing $h$ is equivalent to minimizing $g$. The distinction between
smaller-$g$ and larger-$g$ tie-breaking also appears in incremental
search: \citet{likhachev2005generalized} generalize Lifelong Planning
\astar{} to support larger-$g$ tie-breaking, among other extensions.
Our comparisons concern ordinary, non-incremental \astar{}.

Let $E_\tau(I;\rho)$ be the set of non-goal states expanded by strategy
$\tau$, and let $X_\tau(I;\rho)$ be its expansion count, for random
outcome $\rho$. Since there are no re-expansions,
$X_\tau=|E_\tau|$. We omit $\rho$ for deterministic runs.
For deterministic strategies, \emph{set dominance} on an instance class
means $E_{\tau_1}(I)\subseteq E_{\tau_2}(I)$ on every instance, with a
strict inclusion on at least one. \emph{Expansion-count dominance}
replaces inclusion by $X_{\tau_1}(I)\leq X_{\tau_2}(I)$; for randomized
strategies, we compare expectations. We establish \emph{pairwise
incomparability} by giving instances with strict count inequalities in
both directions. This also rules out set dominance in either direction
for deterministic strategies; the converse implication need not hold.
An additive gap is unbounded when it exceeds every prescribed integer
on some instance. Expectations are over a strategy's random choices on
a fixed instance, not over a distribution of instances.

\subsection{Compulsory expansions and the final plateau}
\label{sec:background}

Define $f^*(v)=\gstar(v)+h(v)$, where the superscript refers to the
optimal cost from the initial state, not to use of a perfect heuristic.
For reachable non-goal states, let
\[
\begin{split}
  S_< &=\{v\notin\mathcal G:f^*(v)<\Cstar\},\\
  S_= &=\{v\notin\mathcal G:f^*(v)=\Cstar\},\qquad
  S_{\leq}=S_<\cup S_=.
\end{split}
\]
The states in $S_<$ are \emph{surely expanded}. We call $S_=$ the
non-goal final $f$-layer and use \emph{final plateau} for search at
$f=\Cstar$, including any eligible optimal goals. We restate the
classical expansion bounds in this notation
\citep{dechter1985optimality}.

\begin{samepage}
\begin{lemma}[\astar{} expansion bounds; after \citet{dechter1985optimality}]
\label{lem:f-bounds}
Every run of the modeled \astar{} algorithm satisfies
\[
  S_<\subseteq E_\tau(I;\rho)\subseteq S_{\leq}.
\]
All states in $S_<$ are expanded before any state in $S_=$.
\end{lemma}
\end{samepage}

\begin{proof}
Consistency gives the optimal-$g$ property at expansion and nondecreasing
selected $f$-values. Before termination, a frontier state on an optimal
solution path is in \Open{} with $f\leq\Cstar$, so no state with
$f^*>\Cstar$ is expanded. For $v\in S_<$, $f^*$ is nondecreasing along
a shortest path to $v$. Each state on that path has $f^*<\Cstar$ and
must be generated with its optimal $g$ and expanded before a goal, whose
$f$-value is at least $\Cstar$. The same reasoning forces these states
before any expansion at $f=\Cstar$.
\end{proof}

The distinction between compulsory and final-layer expansions matters
when interpreting comparisons of search effort: comparing $S_<$ alone
does not determine total expansion counts. In particular,
\citet{holte2010misconceptions} shows that a more accurate heuristic need
not reduce the total number of \astar{} expansions. Our pairwise
comparisons hold the heuristic fixed and vary only the specified
tie-breaking criteria.

An instance is \emph{nonpathological} with respect to $h$ if it has an
optimal solution path with $h(v)<\hstar(v)$ at every non-goal state on
the path \citep{dechter1985optimality}. Otherwise it is
\emph{pathological}. Let $\Pi^*$ be the simple optimal paths whose only
goal is their last state, and define
\[
  \ell^*(I)=\min_{\pi\in\Pi^*}|(\pi\setminus\mathcal G)\cap S_=|.
  \]

The following proposition states the classical
minimum-expansion characterization in the model used in this paper.
All runs considered here return an optimal-cost solution;
expansion optimality instead concerns search effort, measured by the
number of non-goal states expanded. We call a run
\emph{expansion-minimal} on $I$ if no other run of the modeled \astar{}
algorithm expands fewer non-goal states on the same instance.

\begin{samepage}
\begin{proposition}[Minimum-expansion characterization;  after \citet{dechter1985optimality}]
\label{prop:classical-optimum}
Every run satisfies
\[
  X_\tau(I;\rho)\geq |S_<|+\ell^*(I),
\]
and an instance-dependent deterministic tie-breaker attains equality.
Hence
\[
  |S_<|+\ell^*(I)
\]
is the minimum expansion count achievable by the modeled \astar{}
algorithm on $I$.

Moreover, the following are equivalent: the instance is nonpathological;
$\ell^*(I)=0$; some run expands exactly $S_<$; every run expands exactly
$S_<$.
\end{proposition}
\end{samepage}

\begin{proof}
The parent-pointer chain of the selected optimal goal is an optimal
path, and each of its non-goal states was expanded.
Lemma~\ref{lem:f-bounds} also forces all of $S_<$, giving the lower bound.
Choose a path attaining $\ell^*$. Consistency makes its states in $S_=$
a final segment. After $S_<$ has been expanded, an instance-dependent
tie-breaker can select the states of this segment in path order, then
select the generated goal. It expands exactly $|S_<|+\ell^*$ states.

On an optimal path, $f^*(v)<\Cstar$ is equivalent to
$h(v)<\hstar(v)$, so nonpathology is equivalent to $\ell^*=0$.
In that case, an optimal goal is generated from $S_<$ and is selected
before any non-goal in $S_=$. Every run therefore expands exactly $S_<$.
Conversely, the parent-pointer path of any run expanding only $S_<$ has
$\ell^*=0$.
\end{proof}

The minimum is over all tie-breaking rules allowed by our \astar{}
model. Such a rule may use instance-specific information, including
information about states that have not yet been generated, and may
differ from the standard strategies studied next. The existence of an
expansion-minimal rule therefore does not establish a dominance relation
between any two fixed strategies.

\section{Pairwise Incomparability}\label{sec:pairwise}

This section establishes unbounded worst-case separation in the number of expansions among the standard tie-breaking rules.

\subsection{A final-plateau construction}

We start with a simple construction with unbounded branching factor. 
The following construction lets us choose insertion order and heuristic values independently, while keeping all competing states tied on the primary $f$-value. This lets us control FIFO/LIFO preferences and low-$h$/high-$h$ preferences separately, while RANDOM samples uniformly from the states left tied.

\begin{samepage}
\begin{lemma}[Star construction]\label{lem:star}
Let  $M\geq1$ be an integer, $C>0$, and values $0<h(u)<C$ for
$U=\{p,b_1,\ldots,b_M\}$.  Construct a tree with initial state $s$,
unique goal $G$, edges $s\to u$ for all $u\in U$, and one further edge
$p\to G$.  Set
\[
\begin{gathered}
 h(s)=C,\qquad h(G)=0,\\
 c(s,u)=C-h(u),\qquad c(p,G)=h(p).
\end{gathered}
\]
Generate $U$ in any prescribed order.  All edge costs are positive; $h$
is admissible, consistent, zero at the goal, and positive at non-goals.
Every state satisfies $f^*=\Cstar=C$.
If a strategy expands $k$ of the $b_i$ before $p$, then $X_\tau=2+k$.
\end{lemma}
\end{samepage}

\begin{proof}
Each consistency inequality holds with equality.  The unique solution
path $s,p,G$ has cost $C$, and $h$ is perfect on it.  Each $b_i$ is a
dead end with $\hstar(b_i)=+\infty$.  After expanding $s$, all members of
$U$ are known non-goals in \OpenMin{}, with $f=C$.
Only expanding $p$ generates a goal.  Goal preference then terminates
search immediately, after $s$, the $k$ dead ends, and $p$.
\end{proof}

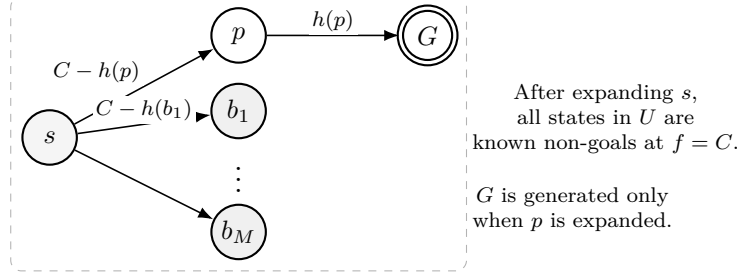
\begin{figure}[htbp]
\centering
\begin{tikzpicture}[x=1cm,y=1cm]
  \node[start] (s) at (0,0) {$s$};
  \node[state] (p) at (2.5,1.35) {$p$};
  \node[dead] (b1) at (2.5,0.35) {$b_1$};
  \node[ell] (dots) at (2.5,-0.45) {$\vdots$};
  \node[dead] (bm) at (2.5,-1.25) {$b_M$};
  \node[goal] (g) at (5.0,1.35) {$G$};
  \draw[edge] (s) -- node[edgelabel,above left] {$C-h(p)$} (p);
  \draw[edge] (s) -- node[edgelabel,above] {$C-h(b_1)$} (b1);
  \draw[edge] (s) -- (bm);
  \draw[edge] (p) -- node[edgelabel,above] {$h(p)$} (g);
  \node[note,anchor=west] at (5.55,0.25) {After expanding $s$,\\all states in $U$ are\\known non-goals at $f=C$.};
  \node[note,anchor=west] at (5.55,-0.95) {$G$ is generated only\\when $p$ is expanded.};
  \begin{scope}[on background layer]
    \node[plateau,fit=(s)(p)(b1)(dots)(bm)(g)] {};
  \end{scope}
\end{tikzpicture}
\caption{The star construction.  Every state has $f^*=C^*$.
After expanding $s$, the $M+1$ states in $U$ are tied non-goals in \Open{}.
Among the states in $U$, only $p$ has a successor, namely the goal $G$.  The expansion count is $2$ plus
the number of dead ends selected before $p$.}
\label{fig:star}
\end{figure}

Consider the nine strategies
\[
\begin{split}
\mathcal P=\{&\FIFO,\LIFO,\RAND,
 \LowH{\FIFO},\LowH{\LIFO},\LowH{\RAND},\\
 &\HighH{\FIFO},\HighH{\LIFO},\HighH{\RAND}\}.
\end{split}
\]

\begin{samepage}
\begin{theorem}[Pairwise incomparability]\label{thm:pairwise}
On the positive-cost trees of Lemma~\ref{lem:star}, for every integer $M\geq1$:
\begin{enumerate}[label=(\roman*),leftmargin=*,itemsep=0.25em]
\item Every pair of distinct deterministic strategies in $\mathcal P$
  has an instance with expansion counts $2$ and $M+2$, and another with
  these counts reversed.
\item Every pair of distinct strategies in $\mathcal P$ involving at
  least one randomized strategy has an instance favoring each strategy
  in expected expansions, with an expected gap of at least $M/2$ in
  each direction.
\end{enumerate}
All comparisons are decided before an optimal goal is generated.
\end{theorem}
\end{samepage}

\begin{proof}
By Lemma~\ref{lem:star}, it suffices to control the number of dead ends
expanded before $p$.  Insertion order and the values $h(u)$ can be
chosen independently.  Group strategies by their heuristic criterion:
none, low-$h$, or high-$h$.

If both strategies have the same heuristic criterion, they differ in
the final FIFO, LIFO, or RANDOM rule.  Give all members of $U$ equal $h$.
Making $p$ first or last in insertion order reverses the deterministic
FIFO/LIFO counts.  Under RANDOM, the rank of $p$ among $M+1$ states is
uniform, so the expected number of preceding dead ends is $M/2$.
This gives both expected inequalities against either deterministic rule.

If one strategy is low-$h$ and the other high-$h$, make $p$ the unique
minimum-$h$ state and all dead ends higher, or make $p$ the unique
maximum-$h$ state and all dead ends lower.  These choices give counts
$2$ and $M+2$ in opposite directions, regardless of the final rule.

It remains to compare a strategy with no heuristic criterion against
one with a low-$h$ or high-$h$ criterion.  If the former is FIFO or LIFO,
choose $p$'s insertion position to make it last for that strategy and
its $h$-value to make it uniquely preferred by the latter.  Reversing
both choices reverses the counts.  If the former is RANDOM, its expected
count is $2+M/2$ independently of the $h$-values; making $p$ uniquely
preferred or disfavored by the heuristic criterion gives the two
expected gaps of $M/2$.

In each deterministic comparison, the expanded sets are $\{s,p\}$
and $\{s,p,b_1,\ldots,b_M\}$, in opposite orders on the two instances.
The randomized comparisons follow from the uniform rank of $p$ in
the cases where randomization affects its position.
\end{proof}

When all members of $U$ have equal $h$, RANDOM assigns a uniform rank
to $p$ among the $M+1$ tied states.  Hence
\[
  X_{\RAND}-2\sim\operatorname{Unif}\{0,\ldots,M\},
  \qquad
  \mathbb E[X_{\RAND}]=2+\frac{M}{2}.
\]
Moving $p$ from first to last in insertion order reverses the FIFO and
LIFO counts but does not change the RANDOM distribution.

\paragraph{Unbounded expansion gaps with bounded branching.} 
While this star construction uses an unbounded branching factor only to generate all
$M+1$ tied states in a single expansion,  this is not essential to
the separation.  The same tied frontier can be generated with maximum
out-degree two by replacing $s$ with a forced chain: each chain state
generates one member of $U$ and the next chain state, with the chain
states assigned $f<C^*$ and the members of $U$ assigned $f=C^*$.
\astar{} therefore expands the entire chain before selecting any member of
$U$.  At that point, the members of $U$ have the prescribed insertion
order and heuristic values, and the remainder of the search is exactly
the star construction.  The forced prefix adds the same number of
expansions to every strategy, so all pairwise additive gaps in
Theorem~\ref{thm:pairwise} remain unchanged.

\subsection{Unbounded gaps on unit-cost grids}
\label{sec:grid-family}
The star construction above gives a simple proof of pairwise incomparability
on an abstract structure, and shows that in general, none of the standard tie-breaking rules dominate the other. 
We next show that a large
tie-breaking effect also occurs under a familiar setting. 
Figure~\ref{fig:grid-pair} illustrates large expansion gaps between
\LowH{\FIFO} and \LowH{\LIFO} on four-neighbor grids with unit costs
and Manhattan distance. Successors are generated in the fixed order
Up, Right, Down, Left in both panels.

In panel~(a), the start has two branches: an L-shaped corridor leading
to the goal, and a cul-de-sac connected to the corridor only through
the start. For an integer $m\geq1$, scale the illustrated layout so
that each leg of the solution corridor has $m+2$ edges and the
cul-de-sac consists of an $(m+1)\times(m+1)$ square plus its entrance
cell. The solution path contains $C_m=2m+4$ non-goal states, including
the start, and the cul-de-sac contains $D_m=(m+1)^2+1$ states.
Every traversable cell is reachable from the start by a monotone path
toward the goal, so $\gstar(v)+h(v)=\Cstar=2m+4$ throughout the grid.

Expanding the start generates the corridor entrance first and the
cul-de-sac entrance second. These states tie in both $f$ and $h$,
so \LowH{\FIFO} enters the corridor, whereas \LowH{\LIFO} enters the
cul-de-sac. Once either entrance is expanded, that branch continues
to offer eligible states with smaller $h$ than the postponed entrance
to the other branch. Consequently, the corridor-first run follows the
solution path to the goal, while the cul-de-sac-first run exhausts
the cul-de-sac before entering the solution corridor. The expansion
counts are therefore $C_m$ and $C_m+D_m$, respectively.

Panel~(b) transposes the grid while retaining the same successor order.
This exchanges the roles of the two entrances, so \LowH{\LIFO} now
expands $C_m$ non-goals and \LowH{\FIFO} expands $C_m+D_m$.
Thus either strategy can incur an additive gap
$D_m=\Theta(m^2)$ and an expansion-count ratio
$(C_m+D_m)/C_m=\Theta(m)$.

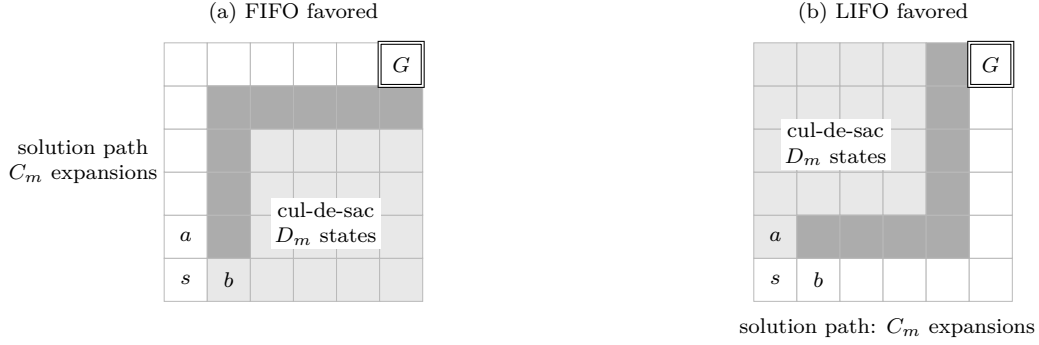
\begin{figure}[tbp]
\centering
\begin{tikzpicture}[x=0.57cm,y=0.57cm,font=\scriptsize]
  \newcommand{\mapA}[1]{%
    \begin{scope}[xshift=#1]
      \foreach \y in {0,...,5}
        {\draw[gridfree] (0,\y) rectangle ++(1,1);}
      \foreach \x in {1,...,5}
        {\draw[gridfree] (\x,5) rectangle ++(1,1);}

      \draw[gridbad] (1,0) rectangle ++(1,1);
      \foreach \x in {2,...,5} {
        \foreach \y in {0,...,3}
          {\draw[gridbad] (\x,\y) rectangle ++(1,1);}
      }

      \foreach \y in {1,...,4}
        {\draw[gridblock] (1,\y) rectangle ++(1,1);}
      \foreach \x in {2,...,5}
        {\draw[gridblock] (\x,4) rectangle ++(1,1);}

      \node at (0.5,0.5) {$s$};
      \node at (0.5,1.5) {$a$};
      \node at (1.5,0.5) {$b$};
      \node at (5.5,5.5) {$G$};
      \draw[double,double distance=0.7pt] (5,5) rectangle ++(1,1);

      \node[gridcallout,above] at (3,6.35) {(a) FIFO favored};
      \node[gridcallout,left] at (-0.05,3.25)
        {solution path\\$C_m$ expansions};
      \node[gridcallout] at (3.75,1.8)
        {cul-de-sac\\$D_m$ states};
    \end{scope}%
  }

  \newcommand{\mapB}[1]{%
    \begin{scope}[xshift=#1]
      \foreach \x in {0,...,5}
        {\draw[gridfree] (\x,0) rectangle ++(1,1);}
      \foreach \y in {1,...,5}
        {\draw[gridfree] (5,\y) rectangle ++(1,1);}

      \draw[gridbad] (0,1) rectangle ++(1,1);
      \foreach \x in {0,...,3} {
        \foreach \y in {2,...,5}
          {\draw[gridbad] (\x,\y) rectangle ++(1,1);}
      }

      \foreach \x in {1,...,4}
        {\draw[gridblock] (\x,1) rectangle ++(1,1);}
      \foreach \y in {2,...,5}
        {\draw[gridblock] (4,\y) rectangle ++(1,1);}

      \node at (0.5,0.5) {$s$};
      \node at (0.5,1.5) {$a$};
      \node at (1.5,0.5) {$b$};
      \node at (5.5,5.5) {$G$};
      \draw[double,double distance=0.7pt] (5,5) rectangle ++(1,1);

      \node[gridcallout,above] at (3,6.35) {(b) LIFO favored};
      \node[gridcallout,below] at (3.1,-0.3)
        {solution path: $C_m$ expansions};
      \node[gridcallout] at (1.9,3.65)
        {cul-de-sac\\$D_m$ states};
    \end{scope}%
  }

  \mapA{0cm}
  \mapB{7.8cm}
\end{tikzpicture}

\caption{ The grid family, shown for 
$m=3$.  White cells form the unique optimal solution path, lightly shaded cells
form the cul-de-sac, and dark cells are blocked.  The successor order is
Up, Right, Down, Left in both panels.  In panel~(a),
\LowH{\FIFO} enters the solution corridor first and expands $10$ states,
whereas \LowH{\LIFO} enters the cul-de-sac first and expands $27$.
Transposition reverses their roles in panel~(b).  The selected goal is
not counted as an expansion.}
\label{fig:grid-pair}
\end{figure}

\section{Minimum- and Maximum-Expansion Tie-Breaking with Locally Perfect Heuristics}\label{sec:near-goal}

The preceding section shows that in general, standard tie-breaking strategies are incomparable, and that large tie-breaking effects persist even
on restricted unit-cost instances.
However, with sufficiently strong constraints, it is possible to find interesting cases where some tie-breaking strategies dominate others.
We now consider one such constraint: 
assume that the heuristic is perfect on specified
low-$h$ layers of the final plateau.
This condition rules out the
mechanism used by the star construction, where dead ends can receive
small heuristic values, and yields strong guarantees for broad classes
of tie-breaking strategies.

The guarantees depend only on whether $h$ is minimized or maximized
within the final plateau. They hold independently of how remaining ties
are resolved. Thus, for an arbitrary final tie-breaker $\tau$, every
low-$h$ strategy \LowH{\tau} minimizes the number of remaining
expansions once the search reaches the relevant perfect region.
Likewise, for an arbitrary final tie-breaker $\sigma$, every high-$h$
strategy \HighH{\sigma} maximizes the total number of expansions under
the weaker assumption that $h$ is perfect on $B_1$.  These statements also hold for every
outcome when the final tie-breaker is randomized. Hence the results
characterize the low-$h$ and high-$h$ classes as a whole, rather than
particular choices such as FIFO, LIFO, or RANDOM within those classes.

Throughout this section, all edges have cost one, and $h$ is consistent,
zero at goals, and strictly positive at non-goals. These are additional
assumptions and are not properties of arbitrary \astar{} instances.
Since $\gstar$ and $\Cstar$ are integers, every state in $S_=$ has a
positive integer heuristic value. For integers $j,k\geq1$, define
\[
 B_j=\{v\in S_=:h(v)=j\},\qquad
 B_{\leq k}=\bigcup_{j=1}^k B_j.
 \]
 Requiring $h$ to be perfect on the final plateau $B_{\leq k}$ imposes no additional
accuracy condition outside that set. In particular, $h$ is perfect
on $B_1$ exactly when every final-plateau state with $h=1$ is a
goal predecessor.

A \emph{search configuration} $H$ includes \Open{}, \Closed{}, current
$g$-values, parent pointers, and insertion timestamps.  At a selection
point with minimum $f=\Cstar$ and no goal in \OpenMin{}, write
\[
 m(H)=\min\{h(v):v\in\OpenMin(H)\}.
\]
Goals with larger $f$ may be present.  Comparisons of continuations
from $H$ hold the entire configuration fixed.
\subsection{\texorpdfstring{Low-$h$}{Low-h} minimizes remaining expansions}

If $v\in B_j$ and $h(v)=\hstar(v)$, consider a successor $w$ on a
shortest path to a goal.  Admissibility and consistency imply
$h(w)=j-1$.  Moreover, $\gstar(w)=\gstar(v)+1$: a cheaper path to $w$
followed by the remaining $j-1$ steps would yield a solution cheaper
than $\Cstar$.  Thus $w\in B_{j-1}$ when $j>1$, and $w$ is a goal
when $j=1$.

\begin{samepage}
\begin{theorem}[Optimal continuation by low-$h$]\label{thm:low-h}
Suppose that $h$ is perfect on $B_{\leq k}$.  Let $H$ be a reachable
configuration with minimum $f=\Cstar$, no goal in \OpenMin{}, and
$m=m(H)\leq k$.  Every low-$h$ continuation \LowH{\tau} performs
exactly $m$ further non-goal expansions.  Every tie-breaking
continuation from the same $H$ performs at least $m$.
Both statements hold for every outcome of any random choices.
\end{theorem}
\end{samepage}

\begin{proof}
  By Lemma~\ref{lem:f-bounds}, all states in $S_<$ have already been expanded. Every current minimum-$f$ entry consequently belongs to
$S_=$ and has its optimal $g$-value.  Low-$h$ selects a state $v\in B_m$.
For any selected $v\in B_j$ with a perfect heuristic value, the descent
property gives a successor $w\in B_{j-1}$, or a goal for $j=1$.
If $w$ were already in \Closed{}, following the same optimal
continuation to its first unexpanded state would exhibit either an
eligible goal or an \Open{} state at $f=\Cstar$ with $h<j$.
Both contradict selection of a minimum-$h$ non-goal with value $j$.
Thus expanding $v$ makes $w$ eligible with its optimal $g$.
Consistency prevents a generated successor from having $h<j-1$;
previously eligible states have $h\geq j$.
Low-$h$ therefore selects heuristic values $m,m-1,\ldots,1$ and then
the generated goal, for exactly $m$ expansions.

For the lower bound, take the parent-pointer path of the optimal goal
selected by any continuation from $H$.  Its first state $w$ unexpanded
at $H$ is in \Open{} with its optimal $g$.  It lies in $S_=$ because
all of $S_<$ has been expanded and it is on an optimal solution path.
Hence $h(w)\geq m$.  Every subsequent non-goal on this parent-pointer
path must be expanded after $H$: parents are expanded before their
children, equal-cost duplicates do not change parents, and no state
is reopened.  The unit-cost suffix from $w$ requires at least
$\hstar(w)\geq h(w)\geq m$ such expansions.
\end{proof}

The theorem compares the number of expansions remaining from a fixed
search configuration, and is about the optimality of a search suffix;
total expansion counts also depend on the work performed before that
configuration is reached.

\subsection{\texorpdfstring{Low-$h$}{Low-h} minimizes total expansions at plateau entry}

Different tie-breakers may reach the final plateau with different histories and insertion orders, but they have expanded the same compulsory set $S_<$ and expose the same relevant plateau states. Therefore, low-$h$'s optimal continuation also minimizes the total expansion count, provided that plateau entry satisfies the local accuracy condition of Theorem~\ref{thm:low-h}.

Let $H_0$ be the first selection point with minimum $f=\Cstar$.
Every run has then expanded precisely $S_<$, although not necessarily
in the same order.  Thus every run reaches $H_0$ after the same number
$|S_<|$ of expansions.

If an optimal goal is present, every run stops.  Otherwise the set of
minimum-$f$ states is independent of the preceding tie-breaking choices.
Indeed, it is

\[
\begin{split}
 R={}&(\{s\}\cap S_=)\\
 &\cup\{v\in S_=:\exists u\in S_<,
       (u,v)\in A,\ \gstar(u)+1=\gstar(v)\}.
\end{split}
\]

All transitions from $S_<$ have been processed, and these are exactly
the final-plateau states generated with their optimal $g$ before any
final-plateau expansion.  Therefore

\[
 m_0=\min_{v\in R}h(v)
\]

is also independent of the preceding expansion order.  Insertion
timestamps need not be the same.

\begin{samepage}
\begin{corollary}[Global low-$h$ optimality at plateau entry]
\label{cor:global-low-h}
Suppose that $H_0$ has no goal in \OpenMin{}, $m_0\leq k$, and $h$ is
perfect on $B_{\leq k}$.  For every final tie-breaker $\sigma$ and every
random outcome,

\[
 \min_\tau X_\tau
   =X_{\LowH{\sigma}}
   =|S_<|+m_0
   =|S_<|+\ell^*(I).
\]

Thus every low-$h$ strategy minimizes the total expansion count.
\end{corollary}
\end{samepage}

\begin{proof}
Every run reaches plateau entry after exactly $|S_<|$ expansions, and
every low-$h$ run then has the same value $m_0$.  Theorem~\ref{thm:low-h}
gives both the lower bound on the number of remaining expansions and
its attainment by every low-$h$ continuation.
Proposition~\ref{prop:classical-optimum} identifies
$m_0=\ell^*(I)$.
\end{proof}

\subsection{\texorpdfstring{High-$h$}{High-h} maximizes total expansions}

The next result needs perfect values only on $B_1$, and characterizes
the complete expanded set up to the identity of one goal predecessor.

\begin{samepage}
\begin{theorem}[Exact high-$h$ expansion count]\label{thm:high-h}
Suppose that $h$ is perfect on $B_1$.  On a nonpathological instance,
every strategy expands exactly $S_<$.  On a pathological instance,
for every final tie-breaker $\sigma$ and every random outcome, there
exists $p\in B_1$ such that
\begin{equation}\label{eq:high-set}
E_{\HighH{\sigma}}=S_<\cup(S_=\setminus B_1)\cup\{p\}.
\end{equation}
Consequently,
\begin{equation}\label{eq:high-count}
X_{\HighH{\sigma}}=
\begin{cases}
|S_<|,&\text{if the instance is nonpathological},\\
|S_<|+|S_=\setminus B_1|+1,&\text{otherwise}.
\end{cases}
\end{equation}
Every high-$h$ strategy maximizes the total expansion count among all
tie-breakers.  For every strategy $\tau$,

\[
 E_\tau\setminus B_1\subseteq E_{\HighH{\sigma}},\qquad
 |E_\tau\cap B_1|=|E_{\HighH{\sigma}}\cap B_1|\in\{0,1\}.
\]

\end{theorem}
\end{samepage}

\begin{proof}
The nonpathological case follows from
Proposition~\ref{prop:classical-optimum}.
Suppose the instance is pathological.  No optimal goal is generated
from $S_<$, since an optimal solution path ending with a predecessor
in $S_<$ would be nonpathological by consistency.
Any final-plateau predecessor of an optimal goal has $h\leq1$ by
consistency and therefore $h=1$ by positivity and integrality.
Every terminating run must thus expand a state in $B_1$.
Conversely, the first expansion in $B_1$ generates an optimal goal
because $h$ is perfect there, and goal preference terminates search.
Every strategy expands exactly one member of $B_1$.

Now consider any $v\in S_=\setminus B_1$, whether or not another strategy
expands it.  Along a shortest path from $s$ to $v$, $f^*$ is
nondecreasing and never exceeds $\Cstar$.
The path cannot pass through a goal before $v$, since unit costs and
$h(v)>0$ would then imply $f^*(v)>\Cstar$.
The path's final segment in $S_=$ has $h$ decreasing by one at every
edge, with all its values at least $h(v)\geq2$.
The first state of this segment is $s$ or is generated with optimal
$g$ from $S_<$.  Before selecting any state in $B_1$, high-$h$ must
expand this state and, inductively, every following state through $v$:
otherwise the first unexpanded state of the segment would remain
eligible with $h\geq2$.
It cannot terminate before its first $B_1$ expansion.  Hence it expands
every $v\in S_=\setminus B_1$.

Lemma~\ref{lem:f-bounds} excludes all other non-goals except the single
selected member of $B_1$, proving~\eqref{eq:high-set}
and~\eqref{eq:high-count}.  Every competing run expands all of $S_<$, some subset
of $S_=\setminus B_1$, and one member of $B_1$ in the pathological case.
This gives count maximality and the asserted inclusions.  When $B_1$
has at most one member, its identity cannot differ between runs.
\end{proof}

Thus the final FIFO, LIFO, or RANDOM criterion can affect \emph{which}
member of $B_1$ high-$h$ expands, but not \emph{how many} states it
expands under these assumptions.

Combining the global low-$h$ result with Theorem~\ref{thm:high-h}
gives the two ends of the expansion range.

\begin{samepage}
\begin{corollary}[Global minimum and maximum expansion counts]\label{cor:global}
Suppose that $H_0$ has no goal in \OpenMin{}, $m_0\leq k$, and $h$ is
perfect on $B_{\leq k}$.  For arbitrary final tie-breakers $\sigma$ and
$\sigma'$, every run satisfies
\begin{align}
\min_\tau X_\tau
&=X_{\LowH{\sigma}}=|S_<|+m_0=|S_<|+\ell^*(I),\label{eq:global-min}\\
X_{\LowH{\sigma}}&\leq X_\tau\leq X_{\HighH{\sigma'}},\label{eq:global-bounds}\\
\max_\tau X_\tau
&=X_{\HighH{\sigma'}}=|S_<|+|S_=\setminus B_1|+1.\label{eq:global-max}
\end{align}
The extrema include all outcomes of randomized strategies.
The low-$h$/high-$h$ gap is unbounded even for each fixed $k\geq1$.
\end{corollary}
\end{samepage}

\begin{proof}
The minimum is Corollary~\ref{cor:global-low-h}.  Since $k\geq1$,
the assumption that $h$ is perfect on $B_{\leq k}$ implies that  $h$ is also perfect on $B_1$, so
Theorem~\ref{thm:high-h} gives the maximum.  The middle inequality
follows from these two extrema.

For an unbounded gap, let $k\geq1$ and introduce $N$ states
$u_1,\ldots,u_N$ with unit-cost edges

\[
 s\to x,\quad s\to u_i,\quad x\to p,\quad u_i\to p,\quad p\to G.
\]

Set $h(s)=h(x)=h(p)=1$, $h(u_i)=2$, and $h(G)=0$.
The heuristic is consistent and positive at non-goals, and $\Cstar=3$.
Here $S_<=\{s,x\}$, $S_{=}=\{p,u_1,\ldots,u_N\}$, and $B_1=\{p\}$.
All final-plateau values are perfect, so the hypothesis holds for the
fixed $k$, and $m_0=1\leq k$.
Every low-$h$ strategy expands $p$ immediately after $s,x$, whereas
every high-$h$ strategy expands all $u_i$ before $p$.
Their counts are $3$ and $N+3$.
\end{proof}

\paragraph{Scope of the guarantees.}
The high-$h$ guarantee can fail without exact heuristic values on
$B_1$, as the unit-cost version of the star construction shows,
or when final-plateau non-goals with $h=0$ are allowed.
The unit-cost assumption is also substantive: even with $h=\hstar$,
minimizing remaining cost need not minimize remaining expansions
when edge costs vary. For example, consider two paths from $s$ to
a unique goal, each costing $2M$, where $M\geq2$. One begins with
an edge of cost $M$ followed by $M$ unit-cost edges; the other
has two edges of costs $1$ and $2M-1$. After expanding $s$,
low-$h$ follows the first path and requires $M$ further expansions,
whereas high-$h$ selects the goal predecessor on the second path
and requires only one.

\subsection{Locally perfect heuristics on grids}\label{sec:manhattan}

The preceding results are stated in terms of heuristic values on the low-$h$
layers of the final plateau.  Four-neighbor grid pathfinding with Manhattan-distance heuristic provides a familiar setting in which these conditions have a simple
geometric interpretation.
In particular, the assumption required for the high-$h$ result---that $h$ is perfect on $B_1$---holds automatically for Manhattan distance.
The stronger conditions
needed for the low-$h$ results depend on the local obstacle structure near the
goal.  %

With a single goal $G=(x_G,y_G)$ and unit-cost orthogonal moves, use
\[
 h(x,y)=|x-x_G|+|y-y_G|.
\]
Every move changes Manhattan distance by one, so $h$ is consistent;
every path to the goal requires at least $h$ moves, so it is admissible.
Every traversable state with $h=1$ is adjacent to the goal.  In
particular, $h$ is perfect on $B_1$, irrespective of obstacles farther
away.

Every such instance is pathological: the last non-goal on each optimal
path has $h=\hstar=1$.  Theorem~\ref{thm:high-h} therefore implies
that high-$h$ expands every non-goal state
with $f^*\leq\Cstar$ except all but one of the final-plateau goal
neighbors.  If the goal has at most one traversable neighbor,
the high-$h$ expanded set contains that of every tie-breaker.

More generally, a final-plateau cell with $h\leq k$ has a perfect
value whenever it has a traversable monotone path to the goal.
An obstacle-free Manhattan ball of radius $k$ is a sufficient, but
not necessary, condition for the heuristic being perfect on $B_{\leq k}$.
Figure~\ref{fig:manhattan-layers} illustrates the difference between
perfect accuracy  at $h=1$ and at larger values.
From any common configuration satisfying $m(H)\leq k$,
Theorem~\ref{thm:low-h} gives minimum \emph{remaining} expansions. 
When this condition holds at $H_0$, the guarantee yields global expansion optimality.  

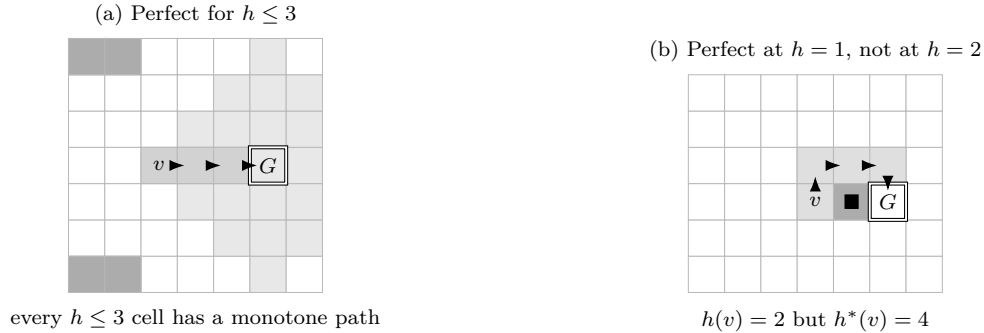
\begin{figure}[htbp]
\centering
\begin{tikzpicture}[x=0.48cm,y=0.48cm,font=\scriptsize]
  \begin{scope}
    \foreach \x in {0,...,6} {
      \foreach \y in {0,...,6} {
        \pgfmathtruncatemacro{\d}{abs(\x-5)+abs(\y-3)}
        \ifnum\d<4
          \draw[gridbad] (\x,\y) rectangle ++(1,1);
        \else
          \draw[gridfree] (\x,\y) rectangle ++(1,1);
        \fi
      }
    }
    \foreach \c in {(0,0),(0,6),(1,0),(1,6)} {
      \fill[gray!65] \c rectangle ++(1,1);
      \draw[gray!55] \c rectangle ++(1,1);
    }
    \foreach \c in {(2,3),(3,3),(4,3)} {
      \fill[gray!35] \c rectangle ++(1,1);
      \draw[gray!55] \c rectangle ++(1,1);
    }
    \node at (2.5,3.5) {$v$};
    \node at (5.5,3.5) {$G$};
    \draw[double,double distance=0.7pt] (5,3) rectangle ++(1,1);
    \draw[edge] (2.8,3.5)--(3.2,3.5);
    \draw[edge] (3.8,3.5)--(4.2,3.5);
    \draw[edge] (4.8,3.5)--(5.2,3.5);
    \node[gridcallout,above] at (3.5,7.25)
      {(a) Perfect for $h\leq3$};
    \node[gridcallout,below] at (3.5,-0.35)
      {every $h\leq3$ cell has a monotone path};
  \end{scope}

  \begin{scope}[xshift=8.2cm]
    \foreach \x in {0,...,6} {
      \foreach \y in {0,...,5} {
        \draw[gridfree] (\x,\y) rectangle ++(1,1);
      }
    }
    \fill[gray!65] (4,2) rectangle ++(1,1);
    \draw[gray!55] (4,2) rectangle ++(1,1);
    \foreach \c in {(3,2),(3,3),(4,3),(5,3)} {
      \fill[gray!25] \c rectangle ++(1,1);
      \draw[gray!55] \c rectangle ++(1,1);
    }
    \node at (3.5,2.5) {$v$};
    \node at (5.5,2.5) {$G$};
    \draw[double,double distance=0.7pt] (5,2) rectangle ++(1,1);
    \draw[edge] (3.5,2.8)--(3.5,3.2);
    \draw[edge] (3.8,3.5)--(4.2,3.5);
    \draw[edge] (4.8,3.5)--(5.2,3.5);
    \draw[edge] (5.5,3.2)--(5.5,2.8);
    \node at (4.5,2.5) {$\blacksquare$};
    \node[gridcallout,above] at (3.5,6.25)
      {(b) Perfect at $h=1$, not at $h=2$};
    \node[gridcallout,below] at (3.5,-0.35)
      {$h(v)=2$ but $\hstar(v)=4$};
  \end{scope}
\end{tikzpicture}
\caption{Perfect Manhattan-distance values near the goal.  In panel~(a), every
state in the shaded radius-$3$ neighborhood has a path that decreases $h$ by
one at each step.  In panel~(b), an obstacle forces a detour from $v$: the
heuristic is perfect at $h=1$, but not at $h=2$.
}
\label{fig:manhattan-layers}
\end{figure}

\section{Heuristic Scaling and Tie-Breaking Sensitivity}\label{sec:scaling}
The preceding results concern ordinary \astar{}, with $f=g+h$.
We now connect heuristic scaling and weighted \astar{} with the nonpathology
criterion of \citet{dechter1985optimality} to identify a
sharp boundary in tie-breaking sensitivity.  With a consistent heuristic that is positive at non-goals, underweighting the heuristic makes the expanded set independent
of tie-breaking, whereas at unit weight the expansion gaps can be
arbitrarily large.

Consider weighted \astar{} using the evaluation function 
\[
 f_\alpha=g+\alpha h,\qquad 0\leq\alpha\leq1,
\]
and retain the search conventions of Section~\ref{sec:model}, with
$f_\alpha$ as the primary key.  In particular, a goal is selected whenever
one is tied for minimum $f_\alpha$, and selecting a goal is not counted
as an expansion.  Let $E_{\alpha,\tau}(I;\rho)$ denote the non-goal
expanded set, and define
\[
 S_<^{(\alpha)}
 =\{v\notin\mathcal G:\gstar(v)+\alpha h(v)<\Cstar\}.
\]
The endpoints $\alpha=0$ and $\alpha=1$ are uniform-cost search
\citep{Dijkstra59} and ordinary \astar{}, respectively.

\subsection{Independence below unit weight}

\begin{samepage}
\begin{corollary}[Expanded-set independence under reduced heuristic weight]
\label{prop:scaling}
Suppose that the original heuristic $h$ is consistent, zero at goals,
and strictly positive at every non-goal.  For every fixed
$0\leq\alpha<1$, every tie-breaker $\tau$, and every random outcome $\rho$,
\begin{equation}\label{eq:scaled-expanded-set}
 E_{\alpha,\tau}(I;\rho)=S_<^{(\alpha)}.
\end{equation}
Thus both the expanded set and the expansion count are independent
of tie-breaking.
\end{corollary}
\end{samepage}

\begin{proof}
Write $q=\alpha h$.  Consistency of $h$ and nonnegative edge costs give
\[
 q(u)\leq\alpha c(u,v)+q(v)\leq c(u,v)+q(v)
 \qquad ((u,v)\in A),
\]
so $q$ is consistent and zero at goals.  At every non-goal on an
optimal solution path, positivity and admissibility give
\[
 q(v)=\alpha h(v)<h(v)\leq\hstar(v).
\]
The instance is therefore nonpathological with respect to $q$, and
Proposition~\ref{prop:classical-optimum} yields
\eqref{eq:scaled-expanded-set}.
\end{proof}

The star construction of Lemma~\ref{lem:star} shows that the boundary
at $\alpha=1$ is sharp in the worst case.  Every non-goal in that
construction satisfies
\[
 \gstar(v)+\alpha h(v)=C-(1-\alpha)h(v)<C=\Cstar
 \qquad(0\leq\alpha<1),
\]
so every strategy expands all $M+2$ non-goals below unit weight.
On the separating instances of Theorem~\ref{thm:pairwise}, the
compared deterministic strategies have counts $2$ and $M+2$ at
$\alpha=1$.
Thus an arbitrarily small reduction in weight can eliminate an
arbitrarily large tie-breaking gap.
Tie-breaking independence at each fixed $\alpha<1$ does not imply that the common expanded set is independent of $\alpha$; it may shrink as $\alpha$ increases.

\subsection{Expanded-set comparison}

The independence result also gives a comparison with ordinary
\astar{} and uniform-cost search.

\begin{samepage}
\begin{corollary}[Expanded-set inclusion under reduced heuristic weight]
\label{prop:scaling-inclusion}
Under the assumptions of Corollary~\ref{prop:scaling}, write $E_\alpha$
for the common expanded set at weight $0\leq\alpha<1$.  For every
ordinary \astar{} tie-breaker $\tau$, every random outcome $\rho$, and
every $0\leq\alpha<1$,
\begin{equation}\label{eq:scaling-sandwich}
 E_{1,\tau}(I;\rho)\subseteq S_{\leq}\subseteq E_\alpha\subseteq E_0
 =\{v\notin\mathcal G:\gstar(v)<\Cstar\}.
\end{equation}
Moreover,
\begin{equation}\label{eq:scaling-monotonic}
 0\leq\alpha<\beta<1\quad\Longrightarrow\quad
 E_\beta\subseteq E_\alpha.
\end{equation}
\end{corollary}
\end{samepage}

\begin{proof}
Lemma~\ref{lem:f-bounds} gives $E_{1,\tau}(I;\rho)\subseteq S_{\leq}$.
For every $v\in S_{\leq}$, positivity implies
\[
 \gstar(v)+\alpha h(v)
 =\gstar(v)+h(v)-(1-\alpha)h(v)<\Cstar,
\]
so $v\in E_\alpha$ by Corollary~\ref{prop:scaling}.
For $0\leq\alpha<\beta<1$, positivity also gives
\[
 \gstar(v)+\beta h(v)<\Cstar
 \quad\Longrightarrow\quad
 \gstar(v)+\alpha h(v)<\Cstar.
\]
This proves \eqref{eq:scaling-monotonic}, and hence $E_\alpha\subseteq E_0$;
the formula for $E_0$ follows by setting $\alpha=0$ in
\eqref{eq:scaled-expanded-set}.
\end{proof}

Reducing the weight makes every non-goal in the ordinary final layer
compulsory.  It therefore removes tie-breaking sensitivity without
saving expansions relative to ordinary \astar{}, but can retain
savings over uniform-cost search.

\paragraph{Example.}
Figure~\ref{fig:scaling-example} illustrates both effects on a
six-state tree.  The only goal is $G$, and $b_1,b_2,x$ are dead ends.
The heuristic is consistent, zero at $G$, and positive at every
non-goal.  The optimal path is $s,p,G$, with $\Cstar=3$; fix the
successor order at $s$ to be $p,b_1,b_2,x$.

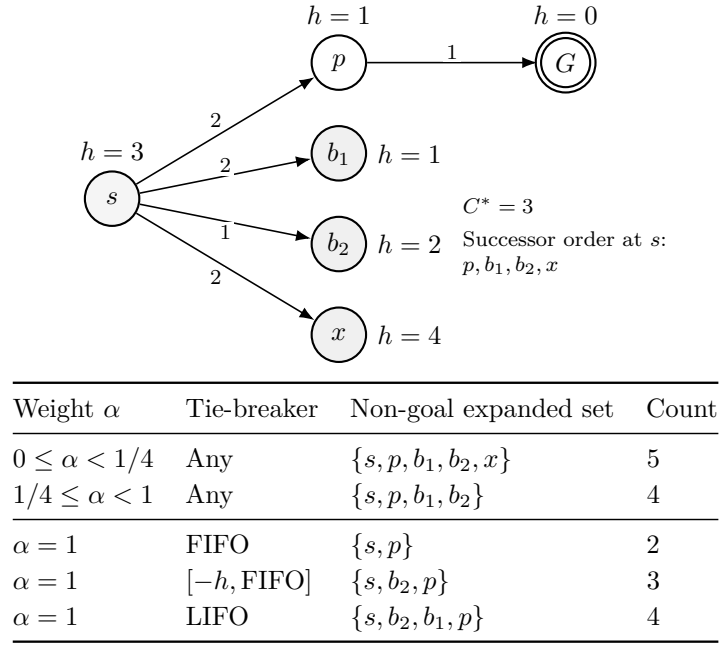
\begin{figure}[htbp]
\centering
\begin{tikzpicture}[
  x=1cm,y=1cm,
  every label/.style={font=\small}
]
  \node[start,label=above:{$h=3$}]
    (s) at (0,0) {$s$};
  \node[state,label=above:{$h=1$}]
    (p) at (3,1.8) {$p$};
  \node[dead,label=right:{$h=1$}]
    (b1) at (3,0.6) {$b_1$};
  \node[dead,label=right:{$h=2$}]
    (b2) at (3,-0.6) {$b_2$};
  \node[dead,label=right:{$h=4$}]
    (x) at (3,-1.8) {$x$};
  \node[goal,label=above:{$h=0$}]
    (G) at (6,1.8) {$G$};

  \draw[edge] (s) -- node[edgelabel,above left] {$2$} (p);
  \draw[edge] (s) -- node[edgelabel,above] {$2$} (b1);
  \draw[edge] (s) -- node[edgelabel,below] {$1$} (b2);
  \draw[edge] (s) -- node[edgelabel,below left] {$2$} (x);
  \draw[edge] (p) -- node[edgelabel,above] {$1$} (G);

  \node[note,anchor=west,align=left] at (4.6,-0.5)
    {$\Cstar=3$\\[3pt]
     Successor order at $s$:\\
     $p,b_1,b_2,x$};
\end{tikzpicture}

\medskip
{\small
\renewcommand{\arraystretch}{1.12}
\begin{tabular}{@{}llll@{}}
\toprule
Weight $\alpha$ & Tie-breaker & Non-goal expanded set & Count \\
\midrule
$0\leq\alpha<1/4$ & Any & $\{s,p,b_1,b_2,x\}$ & $5$ \\
$1/4\leq\alpha<1$ & Any & $\{s,p,b_1,b_2\}$ & $4$ \\
\midrule
$\alpha=1$ & \FIFO & $\{s,p\}$ & $2$ \\
$\alpha=1$ & \HighH{\FIFO} & $\{s,b_2,p\}$ & $3$ \\
$\alpha=1$ & \LIFO & $\{s,b_2,b_1,p\}$ & $4$ \\
\bottomrule
\end{tabular}
}
\caption{Heuristic savings without tie-breaking sensitivity below unit
weight, and distinct expansion counts at unit weight.  Edge labels
are costs; state labels give $h$; shaded leaves are dead ends.
All rows use goal preference, and selecting $G$ is not counted as an
expansion.}
\label{fig:scaling-example}
\end{figure}

After expanding $s$, the priorities are
\[
 f_\alpha(p)=f_\alpha(b_1)=2+\alpha,\qquad
 f_\alpha(b_2)=1+2\alpha,\qquad f_\alpha(x)=2+4\alpha.
\]
For $\alpha<1$, the first three states must be expanded, whereas $x$
is expanded exactly when $\alpha<1/4$.  At $\alpha=1/4$, goal preference
excludes $x$.  Importantly, $p$ and $b_1$ remain tied: at $\alpha=1/2$,
selecting $p$ first generates $G$ with $f_\alpha=3$, but $b_1$ must
still be expanded because $f_\alpha(b_1)=5/2<3$.
Expanded-set independence does not require the absence of ties.

At $\alpha=1$, the states $p,b_1,b_2$ instead all have $f=3$, whereas
$f(x)=6$.  FIFO selects $p$; high-$h$ followed by FIFO selects $b_2,p$;
and LIFO selects $b_2,b_1,p$.  The goal is selected immediately after
$p$, giving the three counts in Figure~\ref{fig:scaling-example}.

\section{Related Work}\label{sec:related-work}

\paragraph{Plateau search.}
\citet{benton2010plateaus} identify $g$-value plateaus as a difficulty
for temporal planning when search steps do not increase the objective
value.
Plateaus also arise in satisficing search under different evaluation functions. \citet{wilt2014speedy} study heuristic plateaus and local minima in greedy best-first search, while \citet{asai2017exploration} study exploration among and within equal-$h$ plateaus. These $h$-plateaus are distinct from the final $f=C^*$ plateaus considered here.

\citet{asai2016aaai} analyze \astar{} tie-breaking, emphasizing the
large plateaus induced by zero-cost actions, and develop plateau-depth
strategies. \citet{asai2017jair} further interpret \astar{} as a sequence
of satisficing searches within equal-$f$ plateaus and introduce
additional distance-to-go guidance within those plateaus.
Our results complement this work with unbounded pairwise separations
under consistent heuristics and strictly positive costs, including a
separation between low-$h$ FIFO and LIFO on unit-cost Manhattan grids.

\paragraph{Optimal expansion versus fixed-policy comparisons.}
\citet{correa2018optimal} show that common \astar{} tie-breakers need not
achieve optimal expansion. For consistent $h$, their optimal-expansion
rule orders states by $[g+h,g+h^*_{\epsilon}]$, where
$h^*_{\epsilon}$ is the exact remaining cost after adding a sufficiently
small positive constant $\epsilon$ to each transition cost. They also evaluate
practical cost-adaptation variants.
Our pairwise result addresses a different question: whether one of the nine
specified standard  strategies uniformly requires no more expansions than another.
Our unit-cost results additionally identify local conditions under
which entire low-$h$ or high-$h$ classes attain extremal expansion
counts without computing an additional heuristic.

\paragraph{Tie-breaking in greedy best-first search.}
\citet{heusner2017} introduce high-water mark benches to characterize
search behavior in greedy best-first search, including states that
cannot be expanded under any tie-breaker and conditions for compulsory
expansion. \citet{heusner2018bestcase} study best- and worst-case
expansion counts over tie-breaking choices for a fixed state space and
heuristic, and compare FIFO, LIFO, and random tie-breaking. These
analyses share our emphasis on expanded sets and tie-dependent
expansion counts, but concern greedy best-first search rather than
cost-optimal \astar{} on its final $f$-layer.

\paragraph{Heuristic accuracy.}
\citet{helmert2008almostperfect} show that \astar{} can require
exponentially many expansions in planning domains even with heuristics
whose error is bounded by a small additive constant. Their lower bounds
count states strictly below the final $f$-layer, so favorable
final-plateau tie-breaking cannot eliminate this effort. In contrast,
our conditions require exact heuristics values on specified
final-plateau layers. In unit-cost search, exact remaining cost is also
shortest remaining path length. This supports the optimal low-$h$
continuation in Theorem~\ref{thm:low-h}, but does not bound the effort
needed to reach the common configuration. Theorem~\ref{thm:high-h}
requires that the heuristic be perfect only on $B_1$ and characterizes maximum total
expansion counts instead.

\paragraph{Heuristic weighting.}
\citet{pohl1970} studies the relationship between search effort,
heuristic accuracy, and heuristic weighting. Our result in
Section~\ref{sec:scaling} is an application of the nonpathology
criterion of \citet{dechter1985optimality}, restated in
Proposition~\ref{prop:classical-optimum}. Under the stated consistency,
positivity, and goal-preference assumptions, every fixed weight
$0\leq\alpha<1$ makes the expanded set independent of tie-breaking,
whereas unbounded pairwise gaps are possible at $\alpha=1$. This
independence is obtained by making ordinary final-layer states
compulsory, not by reducing expansions relative to ordinary \astar{}.

\paragraph{Expansion-order differences beyond tie-breaking.}

This paper studies differences in expansion order caused by explicit tie-breaking among nodes with equal $f$-value. Similar ordering differences can also arise from other aspects of best-first search. \citet{suzuki2026eagerlazy}  compares eager and lazy duplicate detection in \astar{} and give examples in which duplicate handling changes the order of expansions and produces large differences in expansion count \citep{suzuki2026eagerlazy}. Likewise, in parallel best-first search with a shared \Open, concurrent execution can produce an expansion order different from that of the corresponding sequential search \citep{kuroiwa2020,shimoda2025parallel}. These mechanisms are not tie-breakers in our model,
but they illustrate the broader point that changes in expansion order can have large effects on search effort, and 
in both settings, the divergence constructions are structurally similar to those in Figures~\ref{fig:star} and~\ref{fig:grid-pair}.

\section{Discussion and Conclusions}\label{sec:conclusion}

This paper investigated expansion counts under standard tie-breaking rules for \astar{}.
We showed that on positive-cost final-plateau trees,
none of these nine standard tie-breaking rules dominate the others in general,  
with unbounded additive gaps
in deterministic or expected expansions.
The grid example shows that unbounded additive gaps and expansion-count
ratios between \LowH{\FIFO} and \LowH{\LIFO} also arise on the familiar four-neighbor unit-cost grids with Manhattan distance, with instances favoring
each strategy. %

On the other hand, 
in unit-cost search with $h >0$ at non-goals, assuming accuracy on the low-$h$ part of the final plateau
can yield extremal strategies.  Exact heuristic values on
$B_{\leq k}$ make every low-$h$ strategy (regardless of final tie-breaker) optimal for the remaining search from a
common configuration whose minimum eligible $h$-value is at most $k$.
When this condition holds at plateau entry, the same guarantee yields global expansion optimality. 
Exact heuristic values on $B_1$ make every high-$h$ strategy expansion-maximal,
independently of the final tie-breaker.

Finally, we showed that reducing the heuristic weight below one eliminates tie-breaking sensitivity under the stated assumptions, by makes  all ordinary final-layer states compulsory. %

These results contribute to a theoretical understanding of \astar{} search
on the final plateau. One direction for future work is to identify
additional structural conditions that imply bounded or unbounded expansion
gaps between tie-breaking strategies. Another is to identify broader
conditions under which particular classes of tie-breakers attain minimum
or maximum expansion counts. Finally, while we considered heuristic weights
below one, it would be useful to study tie-breaking among states with equal
weighted evaluation values in weighted \astar{} with heuristic weight
greater than one.

\bibliographystyle{abbrvnat}
\bibliography{tiebreak}
\end{document}